\documentclass[numsec,webpdf,modern,medium,namedate]{oup-authoring-template}

\onecolumn

\graphicspath{{Fig/}}

\theoremstyle{thmstyleone}%
\newtheorem{theorem}{Theorem}%  meant for continuous numbers
\theoremstyle{thmstyletwo}%
\theoremstyle{thmstylethree}%
\newtheorem{definition}{Definition}
\newtheorem{lemma}{Lemma}
\newtheorem{corollary}{Corollary}
\newtheorem{assumption}{Assumption}
\usepackage[doublespacing]{setspace}
\usepackage[fontsize=12pt]{fontsize}
\usepackage{tikz}
\usetikzlibrary{positioning}

\begin{document}

\journaltitle{Incompleteness of AI Safety Verification via Kolmogorov Complexity}
% \DOI{DOI HERE}
\copyrightyear{2026}
\pubyear{XXXX}
% \access{Advance Access Publication Date: Day Month Year}
\appnotes{Original article}

\firstpage{1}

\title[Limits of AI Safety Verification]{Incompleteness of AI Safety Verification via Kolmogorov Complexity}

\author[1,$\ast$]{Munawar Hasan\ORCID{0000-0000-0000-0000}}

\authormark{Munawar Hasan}

\address[1]{\orgdiv{Mechanical and Aerospace Engineering}, \orgname{Michigan Technological University}, \orgaddress{\street{Houghton}, \state{MI}, \country{USA}}}

\corresp[$\ast$]{Address for correspondence. Munawar Hasan, MTU, Houghton, MI, 49931, USA. \href{munawarh@mtu.edu}{munawarh@mtu.edu}}

\abstract{
Ensuring that artificial intelligence (AI) systems satisfy formal safety and policy constraints is a central challenge in safety-critical domains. While limitations of verification are often attributed to combinatorial complexity and model expressiveness, we show that they arise from intrinsic information-theoretic limits.
We formalize policy compliance as a verification problem over encoded system behaviors and analyze it using Kolmogorov complexity. We prove an incompleteness result: for any fixed sound computably enumerable verifier, there exists a threshold beyond which true policy-compliant instances cannot be certified once their complexity exceeds that threshold. Consequently, no finite formal verifier can certify all policy-compliant instances of arbitrarily high complexity.
This reveals a fundamental limitation of AI safety verification independent of computational resources, and motivates proof-carrying approaches that provide instance-level correctness guarantees.}


\maketitle

\section{Introduction}
\label{sec:introduction}
A central challenge in modern artificial intelligence (AI) is guaranteeing that autonomous and intelligent systems behave in accordance with prescribed safety and policy constraints. In safety-critical domains such as autonomous driving, multi-agent perception, and secure decision-making, systems’ outputs must satisfy formally specified properties, including collision avoidance, cross-modal consistency, and adherence to regulatory or operational constraints. These requirements naturally give rise to a fundamental verification problem: given an input-output pair and a policy specification, can one certify that the system’s behavior is compliant?

Existing approaches to AI safety and verification largely rely on empirical validation, statistical testing, and formal checking of specific properties. While these methods have achieved significant progress, their limitations are typically attributed to practical factors such as combinatorial explosion of environments, continuous state spaces, and the complexity of learned models. In this work, we argue that such limitations are not merely practical, but fundamental.

Recent work has established information-theoretic limitations in AI security and alignment~\citep{vassilev2025robustaisecurityalignment}, drawing connections to incompleteness phenomena rooted in G\"odel’s theorem~\citep{godel1931formal} and its information-theoretic formulation due to Chaitin~\citep{chaitin1974information}. These perspectives suggest that complete safety guarantees may be inherently unattainable in complex systems. In this work, we provide a formal Kolmogorov complexity-based characterization of these limitations in the setting of policy verification, together with concrete implications for the design of verifiable AI systems.

We adopt a formal perspective in which AI behavior is encoded as finite binary strings and policy compliance is expressed as a predicate over these encodings, so that verification corresponds to proving that a given instance satisfies the policy. We analyze the limits of such verification through the lens of Kolmogorov complexity~\citep{li2008introduction}, which provides a canonical measure of the information content of an individual instance.

Our main result establishes an information-theoretic incompleteness result for AI policy verification: for any fixed sound computably enumerable verification system, there exists a threshold beyond which true policy-compliant instances cannot be certified once their Kolmogorov complexity exceeds that threshold. Consequently, no finite formal verifier can capture all policy-compliant instances of arbitrarily high complexity, revealing an intrinsic limitation of verification that is independent of computational resources.

This limitation has important implications for the design of trustworthy AI systems. In particular, it suggests that universally complete verification is unattainable and motivates a shift toward alternative paradigms in which correctness is established through structured, instance-specific evidence rather than exhaustive characterization of all valid behaviors. This perspective aligns with emerging designs in verifiable AI and secure perception systems, such as the Hermes framework~\citep{hasan2026hermessealzeroknowledgeassurance}, where correctness is established through instance-level constraints and verifiable artifacts in the form of succinct proofs rather than global detection. In such systems, participants provide structured outputs that satisfy formal constraints, and correctness is validated through explicit verification procedures.

More broadly, this shift reflects a change in how correctness is established in complex AI systems: from attempting to characterize all valid behaviors, to verifying the validity of individual instances. This distinction becomes especially important in settings where the space of possible inputs and behaviors is large and highly variable, rendering exhaustive verification infeasible. Our incompleteness result provides a theoretical foundation for this design philosophy, showing that no fixed verifier can certify all valid behaviors and motivating a reliance on per-instance verifiability rather than global completeness.

\paragraph{Contributions.}
\begin{itemize}
    \item We formalize AI policy compliance as a verification problem over encoded system behaviors of the form $x = \langle z, y, \Pi \rangle$.
    \item We establish an information-theoretic incompleteness result via Kolmogorov complexity, showing that no fixed formal verifier can certify all policy-compliant instances of sufficiently high complexity.
    \item We connect these theoretical limits to the design of verifiable AI systems, providing a principled justification for proof-carrying and certificate-based approaches to instance-level verification.
\end{itemize}

Our result does not assume incompleteness, but derives it from fundamental information-theoretic limits. The existence of arbitrarily high-complexity instances, together with the finite descriptive capacity of any computable verifier, implies that some true policy-compliant instances cannot be certified within a fixed system. This limitation arises at the level of universal verification rather than instance-level certification, and motivates a shift toward proof-carrying approaches in which correctness is established through instance-specific certificates.

\section{Related Work}
\label{sec:related-work}
\textbf{Incompleteness and algorithmic information theory.}
Our work is rooted in classical results on the limits of formal reasoning, most notably G\"odel’s incompleteness theorems~\citep{godel1931formal} and their information-theoretic formulations via Kolmogorov complexity. Chaitin~\citep{chaitin1974information} showed that any sufficiently strong formal system cannot prove statements asserting that the Kolmogorov complexity of a string exceeds a given threshold. This perspective establishes a fundamental connection between provability and information content, which we adopt in the context of AI policy verification.

\textbf{Kolmogorov complexity in learning and generalization.}
Kolmogorov complexity has long been studied as a theoretical foundation for learning and model selection, including the Minimum Description Length (MDL)~\citep{grunwald2007minimum} principle. Prior work has used description length to characterize generalization, compression, and inductive bias in machine learning. In contrast, our focus is not on learning performance, but on the limits of verifying properties of individual instances. We use Kolmogorov complexity as a measure of instance-level information content to establish limits of formal verification.

\textbf{Formal verification of AI systems.}
A growing body of work studies formal verification~\citep{katz2017reluplexefficientsmtsolver, huang2017safetyverificationdeepneural} of neural networks and autonomous systems, including reachability analysis, constraint solving, and robustness certification. These approaches aim to provide guarantees about system behavior under specified conditions, and are typically evaluated in terms of scalability and coverage~\citep{cohen2019certifiedadversarialrobustnessrandomized}. Our work complements this line of research by identifying a fundamental limitation: even in principle, no fixed formal verifier can certify all policy-compliant instances once their Kolmogorov complexity exceeds a threshold.

\textbf{AI safety and policy compliance.}
Recent work in AI safety has explored methods for ensuring that system outputs satisfy policy constraints, including rule-based filtering, constrained optimization, and post-hoc verification~\citep{amodei2016concreteproblemsaisafety, 8418593}. These approaches often assume that policy-compliant behavior can be characterized or approximated within a sufficiently expressive system. Our results challenge this assumption by showing that complete certification of policy compliance is impossible for any fixed formal verifier, independent of computational resources.

\textbf{Verifiable computation and proof systems.}
There has been significant progress in cryptographic proof systems, including verifiable computation frameworks~\citep{gennaro2010verifiable} and succinct zero-knowledge proofs such as zk-SNARKs~\citep{groth2016snark} and related systems~\citep{ben2019scalable}, which enable efficient verification of complex computations. These techniques provide a mechanism for certifying correctness via explicit proofs rather than exhaustive validation. Our work provides a theoretical justification for such approaches: since universal verification is fundamentally incomplete, proof-carrying paradigms offer a principled alternative by shifting verification to instance-specific certificates.

These approaches address computational and algorithmic aspects of verification, whereas our result identifies a fundamental information-theoretic limitation. We provide an information-theoretic incompleteness result for AI policy verification, showing that these limitations are not merely practical, but arise from intrinsic bounds on information and provability.

\section{Preliminaries and Problem Formulation}

\subsection{Kolmogorov Complexity}

Kolmogorov complexity provides a measure of the information content of an individual object. Let $U$ be a fixed universal prefix-free Turing machine. The prefix Kolmogorov complexity of a string $x \in \{0,1\}^*$ is defined as
\[
K(x) := \min \{\, |q| : U(q) = x \,\}.
\]

\subsection{Policy-Compliance Model}
We model an AI interaction as a finite binary encoding
\[
x = \langle z, y, \Pi \rangle,
\]
where:
\begin{itemize}
    \item $z$ is the input (e.g., scene, observation),
    \item $y$ is the system output (e.g., action, prediction),
    \item $\Pi$ is a formal policy specification, such as collision avoidance constraints, minimum stopping distance requirements, or consistency conditions over system outputs.
\end{itemize}

This encoding represents a single AI interaction instance, capturing the input, the system's output, and the governing policy in a unified binary form. This allows us to treat policy compliance as a property of a single encoded instance.

\begin{definition}[Policy-compliance]
\label{def:policy}
Define the predicate $P : \{0,1\}^* \to \{0,1\}$ by
\[
P(x)=1 \iff \text{output } y \text{ satisfies policy } \Pi \text{ on input } z.
\]
\end{definition}

\subsection{Proof System}

We model verification as the ability of a formal system to certify policy compliance, i.e., to prove statements of the form $P(x)$.

\begin{definition}[Formal verifier]
\label{def:proof}
Let $T$ be a sound computably enumerable formal theory. A proof checker 
\[
C_T(\varphi,\pi) : \{0,1\}^* \times \{0,1\}^* \to \{0,1\}
\]
returns $1$ if $\pi$ is a valid proof of $\varphi$ in $T$, and $0$ otherwise.
\end{definition}
This definition models verification as provability within a formal system. In our setting, we are interested in statements of the form $\varphi = P(x)$, where $x = \langle z, y, \Pi \rangle$ encodes an AI interaction, and verification corresponds to the existence of a proof $\pi$ such that $C_T(P(x), \pi) = 1$.

\section{Kolmogorov Complexity Limits of Policy Verification}
We now analyze the limits of policy verification through the lens of Kolmogorov complexity. Our approach proceeds in two steps: first, we show that policy-compliant instances can have arbitrarily high complexity under a mild richness assumption; second, we establish that such instances cannot all be certified by any fixed formal verifier.
\subsection{High-Complexity Policy-Compliant Instances}

\begin{assumption}[Richness]
\label{ass:rich}
There exists a constant $d$ such that for infinitely many $m$,
\[
|\{x \in \{0,1\}^m : P(x)=1\}| \ge 2^{m-d}.
\]
\end{assumption}
This assumption captures the fact that policy-compliant behaviors are not confined to a small, highly structured subset of the space. In practical systems, many distinct configurations satisfy safety constraints, leading to a large and diverse set of valid instances. Importantly, the richness assumption does not require the policy $\Pi$ itself to be complex; rather, it reflects variability in inputs $z$ and outputs $y$, which can give rise to a vast number of distinct compliant instances. Even simple policies may therefore admit many instances with high Kolmogorov complexity.

This assumption is mild and holds in settings where the space of valid system behaviors is large, such as perception-driven or environment-dependent systems, where diverse inputs can yield many distinct policy-compliant outputs.

\begin{lemma}[High-complexity compliant instances]
\label{lem:kc}
Under Assumption~\ref{ass:rich}, for infinitely many $m$, there exists $x$ such that
\[
P(x)=1 \quad \text{and} \quad K(x) \ge m-d.
\]
\end{lemma}

\begin{proof}
Among strings of length $m$, fewer than $2^{m-d}$ have Kolmogorov complexity less than $m-d$. Since at least $2^{m-d}$ satisfy $P(x)=1$, one must satisfy $K(x)\ge m-d$.
\end{proof}

\subsection{Incompleteness of Policy Verification}
We now show that no fixed formal verifier can certify all such high-complexity policy-compliant instances.

\begin{theorem}[Incompleteness of policy verification]
\label{thm:main}
Let $T$ be a sound computably enumerable formal theory sufficiently expressive to represent statements involving $P(x)$ and $K(x) > n$. Then there exists a constant $c_T$ such that for all sufficiently large $n$, no statement of the form
\[
P(x) \wedge K(x) > n
\]
is provable in $T$.
\end{theorem}
\begin{proof}[Proof sketch]
Assume that for arbitrarily large $n$, there exists a string $x$ such that $T$ proves $P(x) \wedge K(x) > n$. One can construct a program that enumerates proofs in $T$ and outputs the corresponding $x$, yielding a description of $x$ of size $c_T + O(\log n)$. This implies $K(x) \le c_T + O(\log n)$, contradicting the assumption that $K(x) > n$ for sufficiently large $n$. A full proof is provided in Appendix~\ref{appendix:proof-theorem-1}.
\end{proof}
This result separates truth from provability in the context of AI policy compliance, showing that some valid behaviors cannot be certified within any fixed formal system.

\begin{corollary}[Limit of complete policy verification]
\label{cor:limit-policy}
No fixed sound computably enumerable formal verifier (i.e., theory $T$) can certify all true policy-compliant instances whose Kolmogorov complexity exceeds any fixed bound.
\end{corollary}

\section{Implications for Verifiable AI}

Theorem~\ref{thm:main} reveals a fundamental limitation of verification-based AI safety: no fixed formal system can certify all valid behaviors of arbitrarily high complexity once their descriptive complexity exceeds a threshold.

This highlights a key distinction between detection-based and proof-based approaches. Detection-based systems attempt to characterize all valid behaviors through global rules, whereas proof-based systems require each instance to provide its own certificate of correctness.

Consider an augmented representation
\[
x = \langle z, y, \Pi, \pi \rangle,
\]
where $\pi$ is a proof that $P(x)=1$. Verification reduces to checking $\pi$, which can be done efficiently using modern succinct proof systems. This paradigm shift avoids the need for universal characterization and instead ensures correctness on a per-instance basis. It provides a principled foundation for building trustworthy AI systems in safety-critical settings.

\section{Adaptation to Autonomous Systems}

Consider an autonomous system where $z$ encodes the environment (objects, velocities, sensor inputs, etc.) and $y$ represents a decision such as a control action or trajectory. Let $\Pi$ encode safety constraints such as collision avoidance or minimum stopping distance. The predicate $P(x)$ then determines whether the decision $y$ is safe in context $z$.

In modern autonomous systems, decision-making may be implemented using learned models, including neural network-based policies, or hybrid architectures combining learning and structured components. Verification is therefore applied either directly to model outputs or through auxiliary mechanisms such as safety filters, reachability analysis, or constraint-checking procedures. While these approaches differ in implementation, they share a common objective: to certify that system behavior satisfies predefined safety constraints.

Theorem~\ref{thm:main} shows that these limitations are not merely computational, but fundamental. Even in idealized settings, there exist safe configurations whose correctness cannot be certified by any fixed formal verifier due to their intrinsic informational complexity.

Figure~\ref{fig:proof_vs_verification} illustrates this limitation, highlighting the gap between validity and certifiability and motivating proof-carrying approaches.

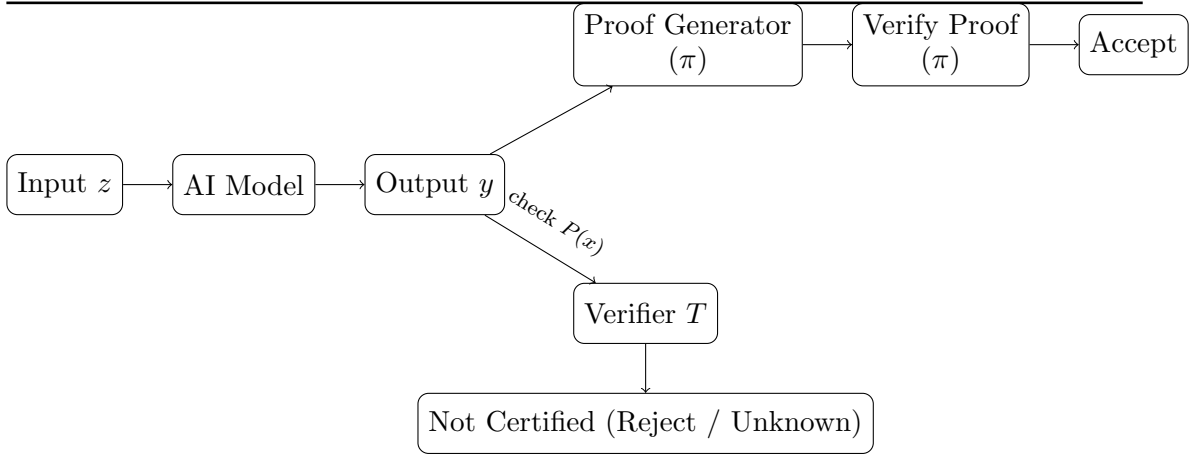
\begin{figure}[t]
\centering
\begin{tikzpicture}[
    node distance=0.65cm and 0.65cm,
    every node/.style={
        draw,
        rounded corners,
        align=center,
        minimum width=1cm,
        minimum height=0.8cm,
        font=\small
    }
]

\node (input) {Input $z$};
\node (model) [right=of input] {AI Model};
\node (output) [right=of model] {Output $y$};

\node (proofgen) [above right=0.9cm and 0.9cm of output] {Proof Generator\\$(\pi)$};
\node (proofcheck) [right=of proofgen] {Verify Proof \\$(\pi)$};
\node (accept) [right=of proofcheck] {Accept};

\node (verifier) [below right=0.9cm and 0.9cm of output] {Verifier $T$};
\node (reject) [below=of verifier] {Not Certified (Reject / Unknown)};

\draw[->] (input) -- (model);
\draw[->] (model) -- (output);

% Proof path
\draw[->] (output) -- (proofgen);
\draw[->] (proofgen) -- (proofcheck);
\draw[->] (proofcheck) -- (accept);

\draw[->] (output) -- node[midway, above, sloped, draw=none, fill=none, font=\scriptsize] {check $P(x)$} (verifier);
\draw[->] (verifier) -- (reject);

\end{tikzpicture}
\caption{Verification Paradigm: A fixed verifier $T$ may fail to certify policy-compliant instances due to fundamental incompleteness, while proof-carrying approaches attach an instance-specific proof $\pi$ enabling verifiable acceptance.}
\label{fig:proof_vs_verification}
\end{figure}

\subsection{Connection to Verifiable Perception Systems}

The incompleteness result in Theorem~\ref{thm:main} suggests a shift in how safety-critical systems should be designed. Rather than attempting to construct globally complete verification mechanisms, systems can instead rely on instance-level certification, where correctness is established through structured, verifiable evidence attached to each output.

This perspective is reflected in emerging designs for verifiable perception and decision-making systems. For example, in multi-vendor perception fusion, contributors provide structured outputs that are validated under shared constraints, while proof-based communication frameworks enable participants to attach succinct cryptographic proofs certifying compliance with safety policies.

These approaches align with the principle that, since no fixed verifier can certify all valid behaviors of arbitrarily high complexity, correctness must be established on a per-instance basis. In this setting, verification becomes the validation of explicit certificates rather than exhaustive characterization of all possible behaviors.

\section{Conclusion}

We established an information-theoretic incompleteness result for AI policy verification, showing that no fixed formal verifier can certify all policy-compliant instances beyond a complexity threshold. This limitation motivates a shift toward proof-based verification paradigms, providing a principled and mathematically grounded direction for the design of verifiable and trustworthy AI systems.

\section*{Acknowledgments}

The author thanks Apostol Vassilev for his work~\citep{vassilev2025robustaisecurityalignment} on information-theoretic limitations in AI security and alignment, which served as a key inspiration for this work.

\begin{appendices}

\section{Appendix: Incompleteness of Policy Verification}

This appendix provides a self-contained proof of Theorem~\ref{thm:main} using standard arguments from Kolmogorov complexity. For completeness, we briefly restate the necessary definitions.

\subsection{Preliminaries}

Let $U$ be a universal prefix-free Turing machine. The prefix Kolmogorov complexity of a string $x$ is
\[
K(x) := \min \{\, |q| : U(q) = x \,\}.
\]

A formal theory $T$ is said to be \emph{sound} if every statement provable in $T$ is true, and \emph{computably enumerable} if the set of its proofs can be effectively enumerated.

\subsection{Counting Bound}

\begin{lemma}[Counting bound]
\label{lem:counting}
For any integer $n$, the number of strings $x$ satisfying $K(x) < n$ is less than $2^n$.
\end{lemma}

\begin{proof}
There are at most $2^n - 1$ binary programs of length strictly less than $n$. Since $U$ is prefix-free, the set of valid programs forms a prefix-free code, and thus there are at most $2^n - 1$ programs of length less than $n$.
\end{proof}

\subsection{High-Complexity Instances}

\begin{lemma}[Existence of incompressible strings]
\label{lem:incompressible}
For every $n$, there exists a string $x$ such that $K(x) \ge n$.
\end{lemma}

\begin{proof}
There are $2^n$ binary strings of length $n$, while fewer than $2^n$ strings can have complexity less than $n$ by Lemma~\ref{lem:counting}. Hence at least one string must satisfy $K(x) \ge n$.
\end{proof}

\subsection{Proof of Theorem~\ref{thm:main}}
\label{appendix:proof-theorem-1}
We restate the theorem:
\begin{theorem}[Restatement of Theorem~\ref{thm:main}]
For any sound computably enumerable theory $T$ sufficiently expressive to represent statements involving both $P(x)$ and $K(x) > n$, there exists a constant $c_T$ such that for all sufficiently large $n$, no statement of the form
\[
P(x) \wedge K(x) > n
\]
is provable in $T$.
\end{theorem}

\begin{proof}
Assume for contradiction that for infinitely many $n$, there exists a string $x$ such that $T$ proves
\[
P(x) \wedge K(x) > n.
\]

Since $T$ is computably enumerable, we can construct a program $A_{T,n}$ that operates as follows:

\begin{enumerate}
    \item Enumerate all candidate proofs in $T$.
    \item For each proof $\pi$, check whether it proves a statement of the form $P(x) \wedge K(x) > n$ for some explicitly represented string $x$.
    \item Extract $x$ from the statement and output it.
\end{enumerate}

Because $T$ is assumed to prove such a statement, the program $A_{T,n}$ eventually halts and outputs a corresponding $x$.

The description length of $A_{T,n}$ consists of:
\begin{itemize}
    \item a fixed proof enumeration procedure,
    \item a fixed description of the theory $T$,
    \item an encoding of the integer $n$, which requires $O(\log n)$ bits.
\end{itemize}

Thus there exists a constant $c_T$ such that
\[
K(x) \le c_T + O(\log n).
\]

On the other hand, since $T$ is sound and proves $K(x) > n$, the statement is true, and hence
\[
K(x) > n.
\]

Combining the inequalities yields
\[
n < K(x) \le c_T + O(\log n),
\]
which is impossible for sufficiently large $n$ since $c_T + O(\log n) < n$ eventually.

Therefore, for sufficiently large $n$, no statement of the form
\[
P(x) \wedge K(x) > n
\]
can be proved in $T$.
\end{proof}

\subsection{Remark}

The proof shows that the limitation arises from a mismatch between the descriptive capacity of the formal system $T$ and the information content of individual instances. This argument is independent of computational constraints and applies to any sound computably enumerable verification system.
\end{appendices}

\bibliographystyle{abbrvnat}
\bibliography{reference}

\end{document}